\documentclass[11pt]{article}

\usepackage[final]{acl}
\usepackage{makecell}

\usepackage{fontawesome5}
\usepackage{times}
\usepackage{latexsym}
\usepackage{pifont}
\usepackage{enumitem}
\usepackage[T1]{fontenc}

\usepackage[utf8]{inputenc}

\usepackage{microtype}

\usepackage{inconsolata}

\usepackage{graphicx}
\usepackage{booktabs} 
\usepackage{amsmath}  
\usepackage{amssymb}  
\usepackage{amsthm}   
\usepackage{algorithm} 
\usepackage{algorithmic} 
\usepackage{url}
\usepackage{booktabs}
\usepackage{siunitx}

\theoremstyle{definition}

\newtheorem{theorem}{Theorem}

\newtheorem{definition}{Definition}
\newtheorem{assumption}{Assumption}

\newtheorem{remark}{Remark}

\newcommand{\E}{\mathbb{E}}
\newcommand{\res}[2]{$#1_{_{\scriptscriptstyle \pm #2}}$}

\title{The Paradox of Outcome Optimization: \\ A Causal Information-Theoretic Bound on Reasoning Shortcuts in LLMs}

\author{
Zihan Chen\thanks{ Equal contribution. \textsuperscript{$\dagger$} Corresponding author. Correspondence to: Zenghui Ding \href{mailto:dingzenghui@iim.ac.cn}{\texttt{<dingzenghui@iim.ac.cn>}}.}\textsuperscript{1,2},
Yiming Zhang\footnotemark[1]\textsuperscript{1,2},
Wenxiang Geng\textsuperscript{1}, \\
\textbf{Zenghui Ding}\textsuperscript{$\dagger$1},
\textbf{Yining Sun\textsuperscript{1,2}} \\
\textsuperscript{1}HFIPS, Chinese Academy of Sciences \quad
\textsuperscript{2}University of Science and Technology of China \\
\vspace{0.35em}
{\small \faGithub\ \textbf{Project:} \href{https://rethinkingrl.github.io/RL_Paradox.github.io/}{RL\_Paradox.github.io}}
}
\begin{document}

\maketitle

\begin{abstract}
Large Language Models (LLMs) aligned via outcome-based Reinforcement Learning (RL) frequently exhibit a critical failure mode: they achieve high performance on in-distribution benchmarks while demonstrating brittle reasoning capabilities on out-of-distribution (OOD) tasks. We term this phenomenon Reward-Induced Manifold Collapse. We establish a theoretical framework bridging Structural Causal Models (SCM) and the Information Bottleneck (IB) principle to explain this paradox. We define reasoning as a high-complexity causal process and shortcut learning as the exploitation of low-complexity spurious correlations. Under the implicit inductive bias of Stochastic Gradient Descent (SGD), models optimized for outcome rewards are biased toward shortcut solutions whenever the training distribution allows for a ``Markovian Screening'' of the true causal mechanism. We derive a new generalization bound based on Semantic Coverage Measure ($\eta$) rather than sample size, showing why data scaling on homogeneous distributions may fail to correct reasoning flaws. We also show that Process Reward Models (PRMs) function as Topological Filters, enforcing step-wise mutual information constraints that render the low-complexity shortcut manifold inadmissible. These results provide a mathematical grounding for the role of process supervision beyond simple credit assignment. 
\end{abstract}

\section{Introduction}

The paradigm of Large Language Models (LLMs), particularly through the pipeline of Pre-training followed by Reinforcement Learning from Human Feedback (RLHF)~\cite{christiano2017deep}, has fundamentally transformed the landscape of artificial intelligence. By aligning probabilistic next-token predictors with human intent, techniques such as Proximal Policy Optimization (PPO)~\cite{schulman2017proximal} have enabled models to exhibit impressive instruction-following capabilities. However, as the field pivots from simple chat capabilities to complex reasoning tasks—such as mathematical derivation, code generation, and logical inference~\cite{wei_chain--thought_2022, kojima2022large}—a critical paradox has emerged.

We observe that models optimized via outcome-based supervision (Outcome Reward Models, ORMs) often achieve high performance on in-distribution (ID) evaluation benchmarks but remain brittle on slightly perturbed, counterfactual, or out-of-distribution (OOD) tasks. For instance, a model may correctly solve a math problem when standard variable names are used but falter when the problem is semantically isomorphic yet textually distinct. This phenomenon suggests that strong benchmark performance may sometimes arise from ``Reward Hacking''~\cite{amodei2016concrete, NEURIPS2022_3d719fee} or ``Shortcut Learning''~\cite{geirhos2020shortcut}, rather than from robust causal reasoning. In such cases, the model appears to fit the reward function through the path of least resistance rather than internalizing the underlying reasoning process.

While this phenomenon is widely recognized empirically, the community lacks a unified theoretical framework for explaining its mechanism. Prevailing hypotheses often attribute this fragility to insufficient data scale or model capacity. Consequently, the standard response is to apply Scaling Laws by increasing the volume of training data ($N$) or the number of parameters~\cite{kaplan2020scaling, 10.5555/3600270.3602446}. In this work, we challenge this view. We argue that the failure of reasoning in outcome-optimized models is not merely a failure of scale, but is rooted in the structure of the optimization objective itself.

\begin{figure*}[t!]
\centering
\includegraphics[width=0.9\linewidth]{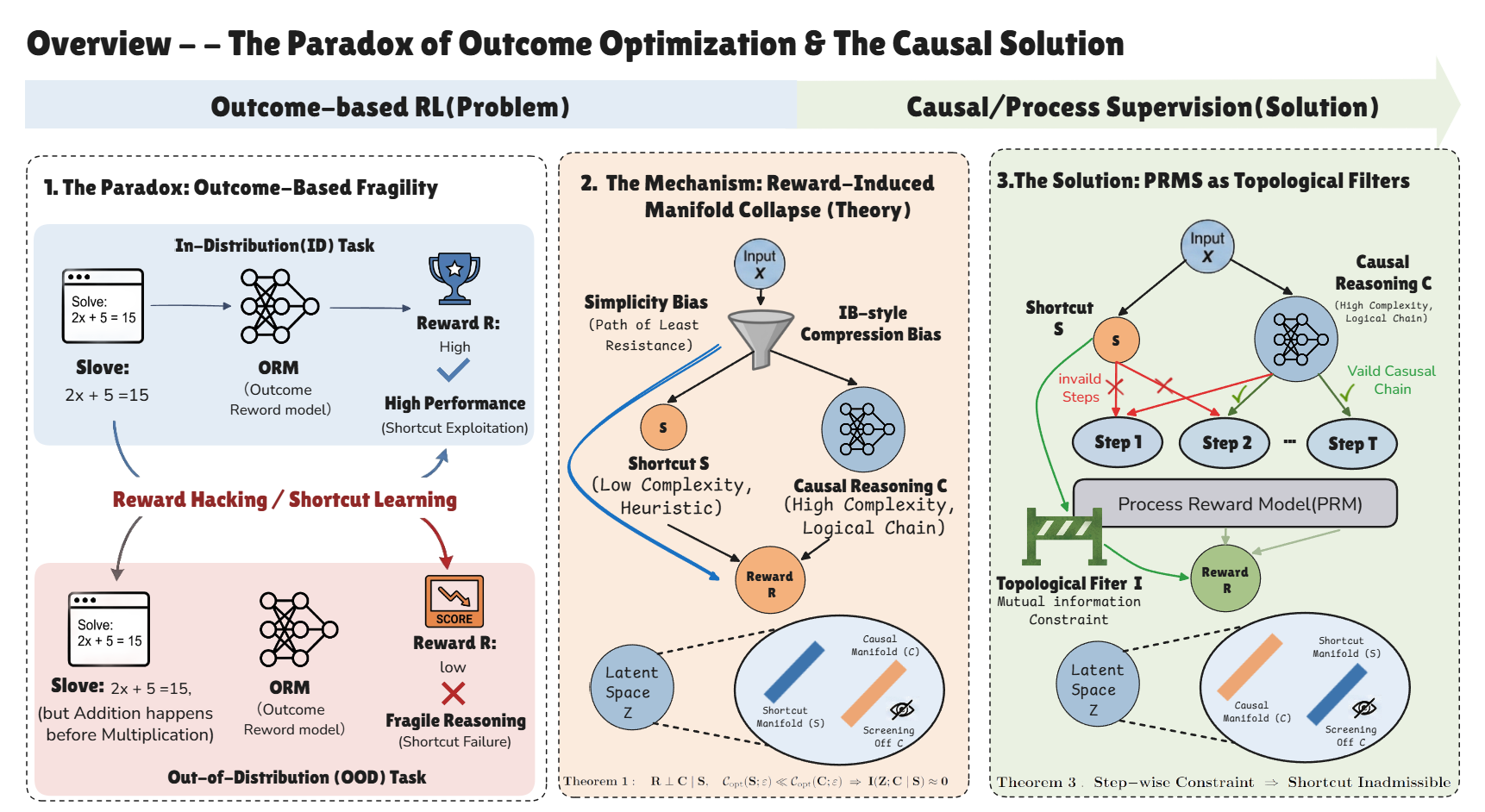}
\caption{\textit{Conceptual overview of Reward-Induced Manifold Collapse and Topological Filtering.} (A) Under outcome optimization (ORM), the learned representation is biased toward the low-complexity shortcut path ($S$) under the simplicity/compression bias of training, while the higher-complexity causal reasoning path ($C$) is screened off. (B) Under process supervision (PRM), step-wise verification acts as a topological filter: shortcut-based trajectories fail intermediate validation, which suppresses the shortcut path and favors the causal reasoning path.}
\label{fig:overview}
\end{figure*}

We propose a theoretical framework grounded in the Information Bottleneck (IB) principle~\cite{tishby2000information, tishby2015deep}, Structural Causal Models~\cite{pearl2009causality}, and algorithmic complexity. We use this framework to characterize LLM training as a compression-driven process that favors reward-sufficient but low-complexity representations. Within this view, we identify a fundamental asymmetry between reasoning features ($C$), which correspond to higher-complexity causal chains, and shortcut features ($S$), which correspond to lower-complexity statistical heuristics.

On standard training distributions, $C$ and $S$ are often aligned. Under the implicit simplicity bias of Stochastic Gradient Descent (SGD)~\cite{10.5555/3305381.3305406, valle2018deep}, optimization can therefore favor $S$ over $C$, since $S$ provides a more efficient compression of the input $X$ for comparable training reward. We term this phenomenon Reward-Induced Manifold Collapse. This perspective helps explain why outcome-based RL can drive models toward brittle heuristics, even when more data is added, so long as the data remains within the same semantic support.

To address this issue, we derive a new generalization bound based on Semantic Coverage rather than sample size, showing that robustness requires breaking the symmetry between $S$ and $C$ through counterfactual data or topological constraints. This also leads to a formal justification for Process Reward Models (PRMs)~\cite{uesato2022solving, lightman2023let}, not merely as a source of denser feedback, but as topological filters that eliminate low-complexity shortcuts from the effective hypothesis space. Our specific contributions are as follows:

\begin{itemize}[leftmargin=*, topsep=3pt, itemsep=2pt]
    \item[\ding{68}] \textbf{The Shortcut Screening Theorem.} We model the reasoning process within a Structural Causal Model (SCM) framework and show that outcome-based optimization can induce \textit{Manifold Collapse}. In particular, when a low-complexity shortcut ($S$) is sufficiently predictive of the reward on the training distribution, the learned representation becomes biased toward shortcut-dominated solutions, thereby suppressing the high-complexity causal feature ($C$).

    \item[\ding{68}] \textbf{Semantic Coverage Generalization Bound.} We derive a generalization bound dependent on the \textit{Semantic Coverage} ($\eta$) of the reasoning topology rather than sample size ($N$). This provides a theoretical basis for the failure of Scaling Laws to resolve reasoning fragility, highlighting the necessity of semantic diversity over mere data quantity.

    \item[\ding{68}] \textbf{Topological Justification for Process Supervision.} We show that Process Reward Models (PRMs) succeed not merely by providing denser feedback, but by enforcing a step-wise mutual-information constraint. This acts as a topological filter that excludes the shortcut manifold from the feasible solution space and thereby favors the learning of causal steps.
\end{itemize}

\section{Related Work}

Our work sits at the intersection of reasoning in Large Language Models (LLMs), Reinforcement Learning from Human Feedback (RLHF), and Information-Theoretic learning principles. We review relevant literature in these domains and highlight the theoretical gaps our work addresses.

\subsection{Reasoning and Chain-of-Thought in LLMs}
The emergence of Chain-of-Thought (CoT) prompting has significantly unlocked the reasoning potential of LLMs \citep{wei_chain--thought_2022}. Subsequent advancements, such as Self-Consistency \citep{wang2022self}, Tree of Thoughts \citep{yao2023tree}, and Graph of Thoughts \citep{10.1609/aaai.v38i16.29720}, have further structured the reasoning process to improve performance on complex mathematical and logical tasks. While these methods focus on \textit{inference-time} interventions, our work investigates the \textit{training-time} dynamics. We ask a fundamental question: why do models optimized via RL fail to internalize these reasoning structures robustly?

\subsection{Alignment, RLHF, and Reward Hacking}
Reinforcement Learning from Human Feedback (RLHF) \citep{NEURIPS2022_b1efde53} has become the standard for aligning LLMs. However, optimizing for a scalar reward is known to be susceptible to ``Reward Hacking,'' where the agent exploits flaws in the reward function \citep{amodei2016concrete}. In the context of reasoning, this manifests as ``Shortcut Learning'' \citep{geirhos2020shortcut}, where models rely on spurious correlations (e.g., lexical overlap) present in the training data. Different from empirical mitigation strategies, we provide a first-principles theoretical explanation, defining reward hacking as the mathematically optimal solution to the Information Bottleneck problem.

\subsection{Outcome vs. Process Supervision}
To combat the sparsity of outcome signals, recent work has shifted towards Process Supervision. \citet{lightman2023let} demonstrated that Process Reward Models (PRMs) significantly outperform Outcome Reward Models (ORMs). While the empirical superiority of PRMs is known, the theoretical mechanism remains under-explored. Current literature largely attributes PRM success to ``denser feedback.'' We offer a deeper, topological perspective: we prove that PRMs function by imposing constraints on the mutual information at each step, physically excluding the low-complexity shortcut manifold. Related feedback-refined paradigms have also been explored in domain-specific agent settings, such as FRAME for medical research workflows~\citep{yu2025frame}.
\section{Preliminaries and Problem Formulation}

To rigorously analyze the phenomenon of Manifold Collapse, we first establish a Structural Causal Model (SCM) for the data generation process and define the optimization landscape of the LLM.

\subsection{The Data Generating Process (DGP)}
We consider a reasoning task defined on a measurable space $\mathcal{X} \times \mathcal{Y}$. We posit that the input $X$ is generated by latent factors.

\begin{definition}[Feature Decomposition]
Let the input $X$ be composed of three latent variables:
\begin{enumerate}[leftmargin=*, labelsep=0.5em]
    \item \textbf{Causal Reasoning Feature ($C$):} The inherent logical structure required to derive the correct answer $Y$ (e.g., the arithmetic operations in a math problem). $C$ is invariant across semantic-preserving transformations.
    \item \textbf{Shortcut Feature ($S$):} Spurious patterns correlated with $Y$ only within a specific training distribution $\mathcal{D}_{train}$ (e.g., specific variable names, sentence length, or lexical triggers).
    \item \textbf{Noise ($\epsilon$):} Irrelevant information.
\end{enumerate}
We assume a structural equation $X = g(C, S, \epsilon)$.
\end{definition}

\begin{assumption}[Training-Distribution Screening and OOD Shift]
\label{assump:correlation}
Let $R \in \{0, 1\}$ be the binary reward signal indicating correctness ($Y_{pred} = Y_{true}$).
\begin{itemize}[leftmargin=*, labelsep=0.5em]
    \item On the training distribution $\mathcal{D}_{train}$, the shortcut feature $S$ is sufficiently predictive of the reward such that, once $S$ is known, the additional contribution of the causal feature $C$ becomes negligible. Formally, we assume
    \[
    R \perp C \mid S.
    \]
    \item On the OOD distribution $\mathcal{D}_{ood}$, this screening relation no longer holds: shortcut features lose predictive power, while the causal feature $C$ remains necessary for correct reasoning and robust generalization.
\end{itemize}
\end{assumption}

\subsection{Model and Optimization Bias}
We model the LLM as a parameterized function $f_\theta: \mathcal{X} \to \mathcal{Z} \to \mathcal{Y}$, where $\mathcal{Z}$ denotes the representation space. The training objective is to maximize reward. However, gradient-based optimization does not only seek reward-maximizing solutions; among solutions with comparable training utility, it may exhibit an implicit preference for features that are easier to acquire.

Instead of relying on the uncomputable Kolmogorov Complexity $K(\cdot)$, we introduce the notion of Optimization Complexity as a tractable proxy for the relative difficulty of learning different feature mappings.

\begin{definition}[Optimization Complexity $\mathcal{C}_{\mathrm{opt}}$]
For a target feature mapping $f$, we define its Optimization Complexity $\mathcal{C}_{\mathrm{opt}}(f;\epsilon)$ as the minimum sample complexity required to learn $f$ up to error $\epsilon$, starting from random initialization, under an optimal learning-rate schedule within a fixed model class. Equivalently, $\mathcal{C}_{\mathrm{opt}}$ can be viewed as a relative measure of optimization difficulty: lower $\mathcal{C}_{\mathrm{opt}}$ means that the feature is easier for gradient-based training to acquire.
\end{definition}

\begin{assumption}[Complexity Asymmetry]
\label{assump:complexity}
We assume that shortcut features are significantly easier to learn than causal reasoning chains:
\begin{equation}
    \mathcal{C}_{\mathrm{opt}}(S;\epsilon) \ll \mathcal{C}_{\mathrm{opt}}(C;\epsilon).
\end{equation}
\textit{Justification:} Shortcut features are often recoverable from shallow or local statistical cues, whereas causal reasoning typically requires multi-step feature composition and therefore higher optimization difficulty.
\end{assumption}

\subsection{The Information Bottleneck Objective}
We use the Information Bottleneck (IB) principle as an analytical lens for understanding the trade-off between reward sufficiency and representational compression. In this view, the learned representation $Z$ can be characterized by the surrogate objective
\begin{equation}
    \min_\theta \mathcal{L}_{IB}(\theta) = \underbrace{-I(Z; R)}_{\text{Maximality}} + \beta \underbrace{I(X; Z)}_{\text{Minimality}} .
\end{equation}
Here, the second term $I(X; Z)$ represents the compression cost. We emphasize that this objective is a conceptual surrogate rather than the exact training loss used in practice. In deep learning, explicit $\beta$ is often absent, but the \textit{implicit regularization} of SGD can be interpreted as an effective $\beta > 0$, penalizing high-complexity representations.

\section{Theoretical Derivations and Proofs}

In this section, we derive three main theorems: (1) The inevitability of Manifold Collapse under Outcome Supervision, (2) The Semantic Coverage Bound showing why scaling data quantity is insufficient, and (3) The Topological Filtering mechanism of Process Reward Models.

\subsection{Theorem 1: The Shortcut Screening Theorem}
We first show that when shortcut features already suffice to explain the training reward, outcome-based supervision can favor them over causal reasoning features that are harder to acquire.

\begin{theorem}[Reward-Induced Manifold Collapse]
\label{thm:collapse}
Consider a training distribution $\mathcal{D}_{\mathrm{train}}$ on which the reward $R$ is conditionally independent of the causal feature $C$ given the shortcut feature $S$, i.e., $R \perp C \mid S$. Suppose further that shortcut features are strictly easier to acquire than causal reasoning features, in the sense that $\mathcal{C}_{\mathrm{opt}}(S;\epsilon) \ll \mathcal{C}_{\mathrm{opt}}(C;\epsilon)$. Then outcome-based training admits a reward-sufficient shortcut-dominated representation $Z^* \approx h(S)$. For such representations, the conditional information carried by $Z^*$ about $C$ beyond $S$ vanishes in the exact case and is small in the approximate case; that is,
\[
\begin{aligned}
I(Z^*; C \mid S) &= 0, && \text{(exact case),}\\
I(Z^*; C \mid S) &\le \varepsilon, && \text{(approximate case).}
\end{aligned}
\]
\end{theorem}

\begin{proof}
The argument proceeds by comparing reward-sufficient representations through their relative optimization complexity.

\noindent \textbf{Step 1: Screening on the training distribution.}
By assumption, on $\mathcal{D}_{\mathrm{train}}$ we have $R \perp C \mid S$. Therefore, once the shortcut feature $S$ is encoded, the additional contribution of the causal feature $C$ to predicting the training reward becomes negligible. In particular, there exists a representation of the form $Z_S = h(S)$ that is sufficient to achieve near-optimal reward prediction on $\mathcal{D}_{\mathrm{train}}$.

\noindent \textbf{Step 2: Optimization-complexity dominance.}
Now compare $Z_S$ with alternative representations that additionally encode $C$. Since these representations achieve comparable training reward on $\mathcal{D}_{\mathrm{train}}$, the relevant distinction lies in their relative optimization complexity. By Assumption~\ref{assump:complexity}, shortcut features satisfy
\[
\mathcal{C}_{\mathrm{opt}}(S;\epsilon) \ll \mathcal{C}_{\mathrm{opt}}(C;\epsilon).
\]
Hence, among representations with comparable training utility, shortcut-dominated solutions are favored by the optimization dynamics because they require lower optimization effort.

\noindent \textbf{Step 3: Screening of the causal feature.}
Consequently, training is biased toward solutions of the form $Z^* \approx h(S)$ rather than representations that additionally encode $C$. For any such shortcut-dominated representation, $Z^*$ carries no additional information about $C$ beyond what is already contained in $S$ in the exact case, so
\[
I(Z^*; C \mid S)=0.
\]
Under approximation error, this becomes
\[
I(Z^*; C \mid S)\le \varepsilon.
\]
This establishes the screening effect.
\end{proof}

\subsection{Theorem 2: The Semantic Coverage Generalization Bound}
This theorem explains why simply adding more data (Scaling Laws) does not fix the problem if the data is distributionally homogeneous.

\begin{definition}[Semantic Support and Coverage]
Let $\Omega$ be the universal space of possible queries. We define the Semantic Quotient Space $\Omega / \sim$, where $x_1 \sim x_2$ if they require the same causal reasoning logic $C$.
Let $\mu$ be a probability measure over $\Omega / \sim$.
We define Semantic Coverage $\eta$ of a training set $\mathcal{D}$ as:
\begin{equation}
\begin{aligned}
\eta(\mathcal{D})
= \mu\Bigl(\bigl\{ [x] \in \Omega/\sim \;\big|\;
\exists (x',y') \in \mathcal{D}, \\
\qquad\qquad S(x') \text{ is not spurious for } [x]
\bigr\}\Bigr).
\end{aligned}
\end{equation}
Simply put, $\eta$ is the portion of the reasoning space where the training shortcuts hold true.
\end{definition}

\begin{theorem}[Coverage-Dependent Generalization Error]
\label{thm:coverage}
For a model trained to convergence on $\mathcal{D}_{train}$ with shortcut feature $S$, the expected risk $\mathcal{R}$ on a robust OOD distribution $\mathcal{D}_{ood}$ is lower bounded by the complement of Semantic Coverage, independent of the sample size $N$.
\end{theorem}

\begin{proof}
Let the model prediction be $\hat{Y} = f(S(X))$. The risk on OOD is:
\begin{equation}
    \mathcal{R}(\mathcal{D}_{ood}) = \E_{x \sim \mathcal{D}_{ood}} [ \ell(f(S(x)), Y(x)) ].
\end{equation}
We partition the OOD space into two regions:
\begin{enumerate}[leftmargin=*, labelsep=0.5em]
    \item $\mathcal{R}_{aligned}$: Regions where shortcut $S$ aligns with causal logic $C$ (Measure $\eta$).
    \item $\mathcal{R}_{misaligned}$: Regions where shortcut $S$ is orthogonal or adversarial to $C$ (Measure $1-\eta$).
\end{enumerate}
\begin{align}
    \mathcal{R}(\mathcal{D}_{ood}) &= \eta \E_{aligned}[\ell] + (1-\eta) \E_{misaligned}[\ell].
\end{align}
In the misaligned region, since $Z \approx S$ and $S \perp Y$ (or $S$ points to wrong $Y$), the loss is lower-bounded by the entropy of the labels plus the KL-divergence of the error:
\begin{equation}
\begin{split}
    \mathbb{E}_{\text{misaligned}}[\ell] &\ge H(Y) + D_{\text{KL}}(P_{\text{true}} || P_{\text{shortcut}}) \\
    &> \delta > 0.
\end{split}
\end{equation}
Thus,
\begin{equation}
    \mathcal{R}(\mathcal{D}_{ood}) \ge (1-\eta) \cdot \delta.
\end{equation}
\textbf{Implication:} Increasing the dataset size $N$ implies sampling more points from $\mathcal{D}_{train}$. If the support of $\mathcal{D}_{train}$ does not expand to cover new semantic variations (i.e., if it merely densifies the existing manifold), $\eta$ remains constant. Therefore, $\lim_{N \to \infty} \mathcal{R}_{ood} \not\to 0$. The error is dominated by the coverage gap $(1-\eta)$.
\end{proof}

\subsection{Theorem 3: PRM as Topological Filter}
We now address why Process Supervision fixes this. We argue that PRMs do not just provide ``more'' signal, but qualitatively ``different'' signal that changes the topology of the solution space.

\begin{assumption}[Non-Markovian Property of Shortcuts]
A reasoning process is a Markov chain $Z_1 \to Z_2 \to \dots \to Z_T$.
A shortcut $S$ is typically a direct mapping $X \to Y$. If forced to generate intermediate steps $Z_t$ based on $S$, the conditional mutual information $I(Z_t; Z_{t+1} | S)$ is low or incoherent, whereas for causal reasoning $C$, $Z_{t+1}$ is deterministically derivable from $Z_t$.
\end{assumption}

\begin{theorem}[Topological Separation via Step-wise Mutual Information]
\label{thm:prm}
Let the objective be the Process Reward maximization $\max \sum_{t=1}^T I(Z_t; R_t)$. If the shortcut feature $S$ cannot generate valid transitions for all $t$ (i.e., it fails the step-wise verification), then the set of optimal parameters $\Theta_{PRM}$ is disjoint from the shortcut manifold $\Theta_{Shortcut}$.
\end{theorem}

\begin{proof}
\textbf{Step 1: The Trajectory Likelihood.}
Under PRM, the model must maximize the joint likelihood of the entire trajectory $\tau = (Z_1, \dots, Z_T)$ being valid.
\begin{equation}
    \mathcal{J}_{PRM} = \sum_{t=1}^T \E [\log P(R_t=1 | Z_t)].
\end{equation}

\noindent \textbf{Step 2: The Shortcut Failure Mode.}
Suppose the model uses shortcut $S$. $S$ predicts the final answer $Y$ well, so $R_T$ is high. However, $S$ contains no information about the intermediate derivation steps $Z_t$ required by the causal logic.
To satisfy intermediate $R_t$, a model relying on $S$ must ``hallucinate'' steps that look plausible.
Let $Z_t^{(S)}$ be steps generated from shortcut $S$. Since $S$ is a low-complexity heuristic (e.g., ``answer is 5''), it lacks the entropy to generate a complex, causally consistent chain.
Thus, for intermediate steps $t < T$:
\begin{equation}
    I(Z_t^{(S)}; R_t) \approx 0 \implies \mathcal{L}_{step\_t} \text{ is high.}
\end{equation}

\noindent \textbf{Step 3: The Causal Advantage.}
Conversely, the causal feature $C$ represents the generative algorithm of the proof. By definition, $C$ induces valid transitions $Z_t \to Z_{t+1}$.
\begin{equation}
    I(Z_t^{(C)}; R_t) \approx H(R_t).
\end{equation}

\noindent \textbf{Step 4: Topological Constraints.}
We compare the total objectives. Let $\lambda$ be the penalty for step-wise error.
\begin{align}
    \mathcal{L}_{ORM}(S) &\approx 0. \quad (\text{Since } S \text{ predicts } Y) \\
    \mathcal{L}_{PRM}(S) &\approx \sum_{t=1}^{T-1} H(R_t) \gg 0. \quad (\text{Step failure})
\end{align}
In the ORM landscape, $S$ is a global minimum (due to low complexity).
In the PRM landscape, $S$ becomes a high-loss region (a ``barrier''). The PRM objective imposes a Step-wise Mutual Information Constraint $\forall t: I(Z_t; R_t) > \epsilon$.
Since $S$ cannot satisfy this, the optimizer is forced to exit the shortcut manifold and traverse the higher-complexity optimization path to find $C$.
Thus, PRM acts as a topological filter that renders the shortcut solution inadmissible.
\end{proof}
\begin{remark}[Robustness to Noisy Verifiers]
A common critique of process supervision is the ``Oracle Fallacy''—the assumption that the reward model itself is perfect. Our framework accommodates this reality. Even if the PRM is trained via outcomes or imperfect data, learning a ``Verification Shortcut'' $S_v$, Theorem \ref{thm:prm} holds under a relaxed condition: Orthogonality of Failure Modes.

Specifically, as long as the Generator's shortcut $S_g$ and the Verifier's shortcut $S_v$ are not perfectly aligned (i.e., their collapse manifolds are disjoint or orthogonal), the joint system prevents ``Joint Manifold Collapse.'' This provides a theoretical justification for ensemble-based PRMs and heterogeneous supervision strategies, which serve to maximize the angle between $S_g$ and $S_v$.
\end{remark}

\section{Experimental Validation}
\label{sec:experiments}

To empirically verify the phenomenon of Reward-Induced Manifold Collapse (Theorem \ref{thm:collapse}), we move beyond synthetic toy tasks to a realistic reasoning evaluation: the Counterfactual MATH Benchmark. This experiment is designed to rigorously distinguish between genuine deductive reasoning and the low-complexity shortcut of reciting pre-trained knowledge.

\subsection{Experimental Setup}

We define the reasoning task under a counterfactual intervention, creating a clear separation between parametric memory and in-context reasoning. The setup maps our theoretical variables as follows:

\vspace{0.5em}
\noindent\textbf{The Shortcut Feature ($S$):}
Parametric Recitation. The standard mathematical rules ingrained in the model's pre-training weights (e.g., the standard Order of Operations \textbf{PEMDAS}, or base-10 arithmetic). Retrieving $S$ has low Optimization Complexity:
\[
\mathcal{C}_{\mathrm{opt}}(S;\epsilon) \approx 0.
\]

\vspace{0.4em}

\noindent\textbf{The Causal Reasoning Feature ($C$):}
In-Context Deduction. The ability to dynamically apply a novel, contrary-to-fact rule provided in the prompt (e.g., a redefined Order of Operations \textbf{PESAMD}, where Addition precedes Multiplication). Deriving answer $Y$ from $C$ requires traversing a high-complexity logic path.

\vspace{0.4em}

\noindent\textbf{Datasets \& Protocol:}
We employ two distinct evaluation sets:

\begin{itemize}[leftmargin=*, topsep=4pt, itemsep=3pt]
    \item \textbf{Standard Set ($\mathcal{D}_{bal}$):} Problems from the MATH benchmark~\cite{hendrycks_measuring_2021} where standard rules apply.
    
    \item \textbf{Counterfactual Set ($\mathcal{D}_{cf}$):} Transformed problems where the axioms are explicitly modified ($c_{\text{real}} \to c_{\text{fake}}$). The counterfactual examples are constructed through an LLM-based generation pipeline; detailed construction procedures and data quality analysis are provided in Appendix~\ref{app:datasets_benchmarks}.
\end{itemize}

A response is marked correct \textit{only} if it strictly adheres to the counterfactual premise, filtering out ``lucky guesses.''

\subsection{Results: Empirical Confirmation of Manifold Collapse}

We evaluate Qwen2.5-3B, Qwen2.5-7B, and Llama-3-8B under both outcome-level supervision (ORM) and process-level supervision (PRM) on the 472 solvable samples (see Appendix~\ref{app:datasets_benchmarks} for dataset statistics). The results, averaged over 5 independent runs, are summarized in Table~\ref{tab:counterfactual-results}.

\begin{table}[h]
    \centering
    \small
    \setlength{\tabcolsep}{5pt}
    \renewcommand{\arraystretch}{1.10}
    \caption{Counterfactual reasoning performance on the 472 solvable samples. Results are averaged over 5 independent runs. Gap is defined as $\mathrm{Std.\ Acc.} - \mathrm{CF\ Acc.}$. Smaller gaps indicate better robustness.}
    \label{tab:counterfactual-results}
    \begin{tabular}{l l c c c}
        \toprule
        \textbf{model} & \textbf{Sup.} & \textbf{Std. Acc.} & \textbf{CF Acc.} & \textbf{Gap} \\
        \midrule
        Qwen & ORM & \res{68.01}{0.42} & \res{38.14}{0.34} & \res{29.87}{0.47} \\
          (3B)    & PRM & \res{\textbf{69.92}}{0.34} & \res{\textbf{58.69}}{0.34} & \res{\textbf{11.23}}{0.52} \\
        \midrule
        Qwen & ORM & \res{79.24}{0.34} & \res{43.64}{0.34} & \res{35.60}{0.36} \\
          (7B)   & PRM & \res{\textbf{81.06}}{0.24} & \res{\textbf{73.35}}{0.28} & \res{\textbf{7.71}}{0.41} \\
        \midrule
        Llama   & ORM & \res{46.44}{0.28} & \res{22.03}{0.15} & \res{24.41}{0.28} \\
          (8B)     & PRM & \res{\textbf{47.80}}{0.24} & \res{\textbf{36.78}}{0.12} & \res{\textbf{11.02}}{0.15} \\
        \bottomrule
    \end{tabular}
\end{table}

\noindent \textbf{Quantitative Analysis.}
The results show a clear and consistent robustness gap under counterfactual rule shifts. Across all three backbones, ORM models suffer substantial drops from standard to counterfactual evaluation: 29.87\% for Qwen2.5-3B, 35.60\% for Qwen2.5-7B, and 24.41\% for Llama-3-8B. In contrast, PRM substantially reduces this gap, bringing it down to 11.23\%, 7.71\%, and 11.02\%, respectively. This pattern suggests that the collapse is not specific to the Qwen2.5 family, and that process supervision consistently improves robustness to rule changes.

\noindent \textbf{Qualitative Mechanism.}
Inspection of the generated Chain-of-Thought (CoT) is consistent with the \textit{Shortcut Screening Theorem}. Even when the model explicitly restates the new rule (e.g., ``Note: Addition is now before Multiplication''), ORM-trained models often revert to standard PEMDAS in subsequent calculation steps.

Mathematically, this behavior is consistent with a collapse regime in which the learned representation $Z$ carries little additional information about the causal logic $C$ once the shortcut feature $S$ is already encoded:
\begin{equation}
    I(Z; C \mid S) \to 0.
\end{equation}

In this regime, the model becomes effectively insensitive to prompt-level logical constraints whenever they conflict with entrenched parametric shortcuts.
\subsection{Micro-Analysis: Probing Manifold Collapse}
\label{sec:probing}

While Table \ref{tab:counterfactual-results} demonstrates the performance collapse, we seek to verify the underlying mechanism: is the causal logic $C$ physically lost in the representation space, or merely unexpressed? To answer this, we conducted a linear probing analysis to visualize the information dynamics layer by layer.

\noindent \textbf{Setup.} We constructed a ``Decoupled Probe Dataset'' where surface symbols ($S$) and causal operations ($C$) are orthogonal (e.g., the symbol ``+'' denotes multiplication). We froze the backbones of both the Qwen2.5-7B and Llama-3-8B ORM/PRM models and trained linear classifiers on the hidden states $Z_l$ at each layer $l$ to predict both $S$ and $C$. 

\begin{figure}[h]
    \centering
    \includegraphics[width=1.0\linewidth]{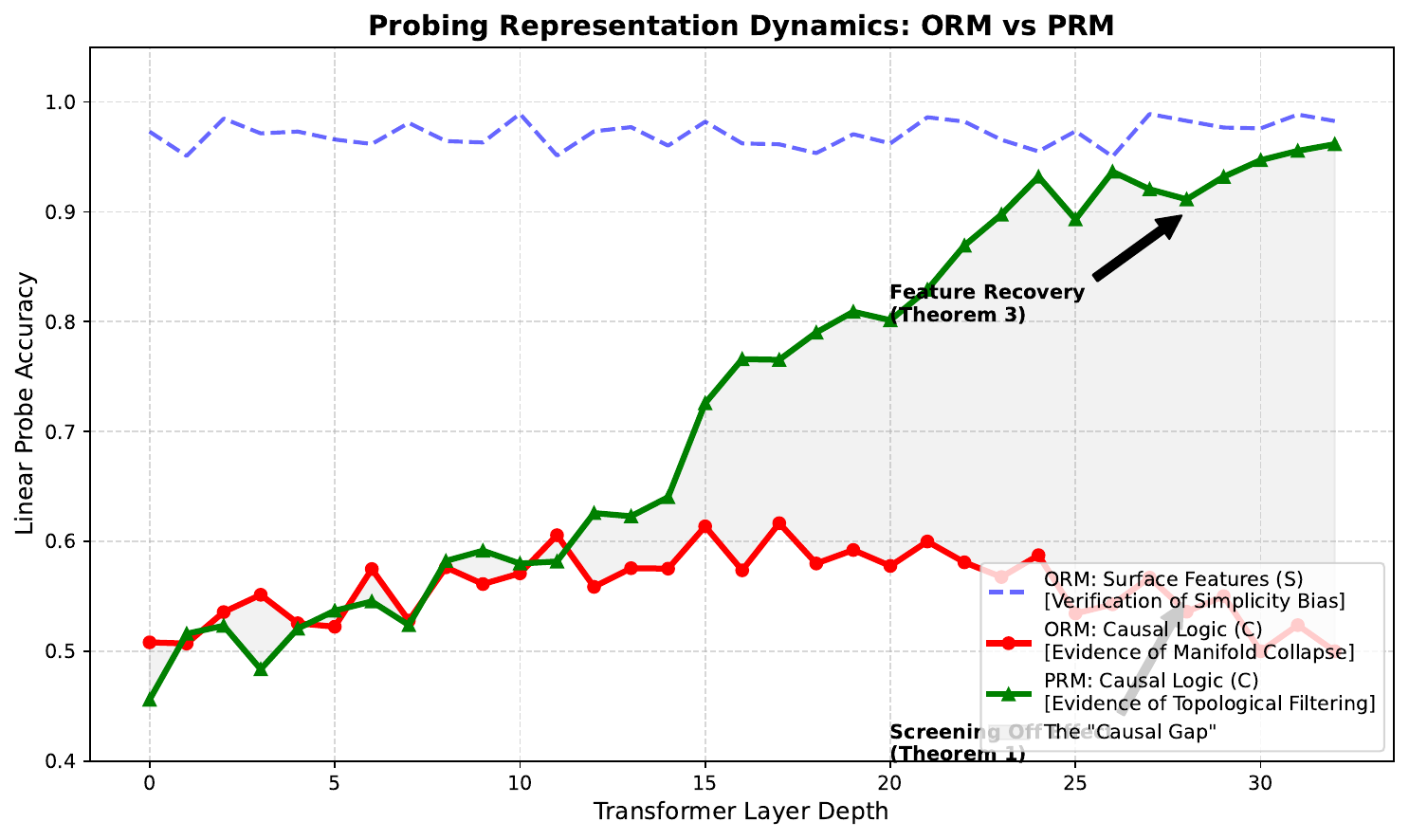}
    \caption{\textbf{The Anatomy of Manifold Collapse.} Linear probe accuracy across model layers. \textbf{Blue (dashed):} ORM captures surface shortcuts ($S$) from early layers. \textbf{Red:} ORM fails to encode the true causal logic ($C$) even in deep layers, consistent with the ``screening off'' effect ($I(Z; C \mid S) \to 0$). \textbf{Green:} PRM promotes the recovery of causal features in deeper layers under step-wise topological constraints.}
    \label{fig:probe_dynamics}
\end{figure}

\noindent \textbf{Results.} The probing dynamics, visualized in Figure \ref{fig:probe_dynamics}, provide further support for our theoretical claims:

\begin{itemize}[leftmargin=*, labelsep=0.5em]
    \item \textbf{Simplicity Bias (Blue Line):} The ORM model achieves near-perfect accuracy on surface features ($S$) starting from the very first layer and maintains this advantage throughout most of the network. This is consistent with $S$ having very low optimization complexity ($\mathcal{C}_{\mathrm{opt}}(S;\epsilon) \approx 0$).
    \item \textbf{Screening Off (Red Line):} Crucially, despite the ORM model's depth, the probe for causal logic ($C$) remains near chance level throughout the network. This provides representational evidence for Theorem \ref{thm:collapse}: outcome-based training biases the representation toward the shortcut manifold, effectively screening off the causal mechanism.
    \item \textbf{Topological Recovery (Green Line):} In contrast, the PRM model shows a phase transition in deeper layers (Layers 16--24), where the causal probe accuracy rises significantly. This is consistent with Theorem \ref{thm:prm}: step-wise verification acts as a topological filter, rendering the shortcut-only path inadmissible and forcing the encoding of $C$.
\end{itemize}

\subsection{Discussion: The Necessity of Topological Filtering}

These findings suggest that Scaling Laws alone do not solve reasoning alignment. Increasing model size (3B $\to$ 7B) improved standard accuracy but also widened the Collapse Gap (29.87\% $\to$ 35.59\%), suggesting that larger models may become better at memorizing shortcuts without necessarily becoming more flexible in reasoning.

This empirical evidence motivates the use of Process Reward Models (PRMs), as discussed in Theorem \ref{thm:prm}. A PRM that verifies step-by-step adherence to the counterfactual axioms can penalize the exact step where the model reverts to memory, thereby topologically blocking the shortcut path and encouraging the optimizer to recover the causal reasoning circuit.

\section{Conclusion}

In this work, we established a theoretical framework for understanding the paradox of outcome optimization in Large Language Models: why models can achieve strong performance on in-distribution benchmarks while exhibiting fragile, non-causal reasoning on out-of-distribution tasks. By bridging Structural Causal Models with the Information Bottleneck principle, we characterized Reward-Induced Manifold Collapse not as a mere artifact of insufficient data, but as a consequence of optimization bias under the constraint of Optimization Complexity.

Our theoretical contributions are threefold. First, the Shortcut Screening Theorem shows that when a low-complexity heuristic ($S$) and a high-complexity causal mechanism ($C$) are aligned with the reward on the training distribution, the simplicity bias of SGD can bias learning toward shortcut-dominated solutions and suppress the causal mechanism. This provides a formal perspective on reward hacking as a consequence of compression-driven optimization. Second, our Semantic Coverage Bound challenges the universality of Scaling Laws for reasoning tasks. We show that generalization error is lower-bounded by the topological coverage of the semantic space ($\eta$), implying that scaling sample size ($N$) on homogeneous distributions may be insufficient to correct reasoning flaws.

Third, we provide a topological justification for process supervision. We show that Process Reward Models (PRMs) succeed not merely by providing denser signals, but by acting as Topological Filters that enforce step-wise mutual information constraints, rendering the shortcut manifold inadmissible. Our analysis of Orthogonal Failure Modes further suggests that PRMs do not require perfect ``Oracle'' verifiers; robustness may be achieved via ensemble-based supervision in which the generator and verifier possess disjoint collapse manifolds.

Looking forward, our findings suggest a broader shift in the training of reasoning agents. Progress toward robust reasoning may require moving beyond the blind scaling of token quantity toward the scaling of Causal Diversity. Future work should focus on: (1) developing metrics for Semantic Coverage to guide data curation; (2) designing ``Heterogeneous Supervision'' objectives that explicitly maximize the orthogonality between generation and verification errors; and (3) investigating architectural inductive biases that can physically decouple System 1 (heuristic) and System 2 (analytic) representations.


\section*{Limitations}

\noindent \textbf{Empirical Scope and Model Scale.} While our theoretical framework is general, our empirical validation currently covers Qwen2.5-3B, Qwen2.5-7B, and Llama-3-8B under counterfactual reasoning evaluation. However, we have not yet investigated whether the same collapse patterns and PRM robustness gains persist in frontier-scale models (e.g., \(>100\)B parameters), where emergent capabilities may alter the relative optimization complexity landscape \(\mathcal{C}_{\mathrm{opt}}(\cdot;\epsilon)\) between shortcut and causal features.

\noindent \textbf{Artificiality of Counterfactual Tasks.} Our experimental definition of a "shortcut" relies on parametric memory of standard axioms (e.g., PEMDAS), which we contrast with "causal reasoning" via explicit rule-reversal. While this provides a rigorous testbed, real-world reasoning shortcuts are often more subtle (e.g., reliance on lexical cues or answer length) and less binary than the clear-cut axiom violations modeled here.

\noindent \textbf{Assumption of Verifiability.} Our proposed solution, the Topological Filter via PRMs, assumes the existence of a verifier capable of calculating step-wise mutual information. In practice, constructing such high-fidelity process supervision signals is expensive and technically challenging, particularly for domains where ground-truth logical steps are ambiguous or unavailable.

\noindent \textbf{Complexity Asymmetry.} Our theoretical analysis hinges on the assumption that shortcut features have lower relative optimization complexity than causal features, i.e., \(\mathcal{C}_{\mathrm{opt}}(S;\epsilon) \ll \mathcal{C}_{\mathrm{opt}}(C;\epsilon)\). In scenarios where spurious correlations are themselves difficult to learn, or where the causal mechanism is comparatively simple, the separation between shortcut- and causal-dominated solutions may weaken, and the filtering effect predicted by our framework may no longer hold as strongly.

\section*{Ethical Considerations}

\noindent \textbf{Reliability in Safety-Critical Applications.} Our theoretical analysis of ``Reward-Induced Manifold Collapse'' underscores a critical vulnerability in current AI alignment paradigms: the tendency of models to maximize rewards via spurious shortcuts rather than genuine causal reasoning. In safety-critical sectors such as healthcare or autonomous control, this ``illusion of competence''—where models perform well in-distribution but fail catastrophically under slight perturbations —poses significant risks. Our work aims to mitigate these risks by mathematically formalizing the necessity of Process Supervision, thereby encouraging the development of systems that are not only performant but mechanistically robust and transparent.

\noindent \textbf{Environmental Sustainability of Research.} We challenge the prevailing reliance on ``Scaling Laws''—the assumption that increasing model size and data volume is the universal solution to reasoning failures. By demonstrating that robustness is lower-bounded by Semantic Coverage rather than sample size, our framework argues against the inefficient expenditure of computational resources on homogeneous data. We advocate for a shift towards data efficiency and topological diversity, which aligns with broader goals of reducing the carbon footprint associated with training large-scale foundation models.

\noindent \textbf{Potential for Dual Use.} While establishing rigorous bounds on reasoning shortcuts improves model reliability, we acknowledge that robust reasoning agents could potentially be leveraged for malicious purposes, such as generating more coherent and persuasive disinformation or automating complex cyber-attacks. However, we believe that the benefits of understanding and controlling the reasoning process via topological constraints outweigh these risks, as these same tools are essential for detecting and filtering adversarial logic.

\noindent \textbf{Use of AI Assistants.} In accordance with standard ethical guidelines, we acknowledge the use of Large Language Models (LLMs) to assist with grammatical error correction and stylistic polishing of the manuscript. We certify that the novelty of the scientific claims, the design of the theoretical framework, and the execution of experiments are the original work of the authors and have not been generated by AI.

\section*{Acknowledgements}
We sincerely thank all the anonymous reviewers and (S)ACs for their constructive comments and helpful suggestions. This work was supported by the National Key Research and Development Program of China (Grant No.~2024YFF0507603) and the Anhui Provincial Major Science and Technology Project (Nos.~202303a07020006 and 202304a05020071).

\bibliography{custom}
\bibliographystyle{acl_natbib}

\appendix

\section{Full Evaluation Setup}
\label{sec:exp_setup}

\subsection{Model Configurations}
\textbf{Outcome and Process Reward Models}
We evaluate Qwen2.5-3B, Qwen2.5-7B, and Llama-3-8B under both outcome-level supervision (ORM) and process-level supervision (PRM). The empirical confirmation experiments in Section~5.2 were implemented in the LLaMA-Factory framework~\cite{zheng2024llamafactory}. The Qwen2.5 models serve as our main controlled backbone pair for comparing collapse under ORM and robustness gains under PRM, while the additional Llama-3-8B results provide evidence that the phenomenon is not specific to a single model family. For inference and evaluation, we adopt the official default hyperparameter settings from the corresponding framework and model implementations, and do not perform additional hyperparameter tuning or search. This choice is intended to ensure straightforward reproducibility and fair comparison across methods.

\noindent \textbf{Linear Probing Setup}
To analyze the internal representations, we employ a linear probing technique. The probing experiments in Section~5.3 were conducted in EasyR1~\citep{zheng2025easyr1}. We freeze the model parameters and train simple linear classifiers on the hidden states of each transformer layer. These probes are designed to distinguish between surface-level shortcut features ($S$) and deep causal reasoning features ($C$). All probing results reported in Section~5.3 are based on a single run for each model.

\subsection{Datasets and Benchmarks}
\label{app:datasets_benchmarks}
We evaluate performance on two distinct variations of mathematical reasoning tasks to measure robustness:

\begin{itemize}[leftmargin=*, labelsep=0.5em]
    \item \textbf{Standard Set ($\mathcal{D}_{bal}$)}: A subset of the MATH benchmark where standard mathematical rules (e.g., PEMDAS) apply. This dataset assesses the model's in-distribution performance.
    
    \item \textbf{Counterfactual Set ($\mathcal{D}_{cf}$)}: A modified dataset where arithmetic rules are explicitly altered in the prompt (e.g., "Addition happens before Multiplication"). This requires the model to ignore pre-trained memory and reason using the new in-context rules.
    
    \item \textbf{Dataset Statistics.}
Our evaluation begins with 500 examples from the MATH test set. Following a data quality audit, we remove unsolvable items and retain 472 solvable examples for final evaluation. The Standard Set ($D_{bal}$) consists of the original test problems restricted to these retained examples. The Counterfactual Set ($D_{cf}$) is constructed from the same examples using \texttt{Gemini 2.5 Pro}, which generates explicit counterfactual rule modifications and rewrites each problem accordingly while preserving the underlying problem structure. Since our study is evaluation-focused, we do not introduce additional train/dev/test splits beyond the original benchmark split.
\end{itemize}

\section{Implementation Details}
All experiments were conducted on a single server equipped with 4 NVIDIA A100 (80GB) GPUs. Our implementation relies on PyTorch and the Hugging Face Transformers library.

\end{document}